\documentclass[11pt]{article}

\usepackage[final]{acl}

\usepackage{eso-pic}
\usepackage{algorithm}
\usepackage{algorithmic}
\usepackage{subcaption}
\usepackage{times}
\usepackage{latexsym}
\usepackage{amssymb}
\usepackage[T1]{fontenc}

\usepackage[utf8]{inputenc}

\usepackage{microtype}

\usepackage{inconsolata}
\usepackage{amsthm}
\newtheorem{definition}{Definition}
\newtheorem{lemma}{Lemma}
\usepackage{graphicx}
\usepackage{float}
\usepackage{dblfloatfix}

\usepackage{multirow} 
\usepackage{booktabs}
\usepackage{colortbl}
\usepackage{xcolor}
\usepackage{amsmath}
\usepackage{pifont}
\newcommand{\cmark}{{\large\ding{51}}}
\newcommand{\xmark}{{\color{red}\ding{55}}}

\title{Nullify: Null-Space Activation Steering for Training-Free LLM Unlearning}

\author{
  \textbf{Wei Zhai}\textsuperscript{1,6}\thanks{Equal contribution.}\quad
  \textbf{Xiang Liu}\textsuperscript{2}\footnotemark[1]\quad
  \textbf{Qiang Huang}\textsuperscript{1}\quad
  \textbf{Rui Qian}\textsuperscript{1}\quad
  \textbf{Lemao Liu}\textsuperscript{1}
  \\
  \textbf{Ziwei Li}\textsuperscript{3}\quad
  \textbf{Ziqi Wang}\textsuperscript{4}\quad
  \textbf{Zhitao Huang}\textsuperscript{5}\quad
  \textbf{Dejing Dou}\textsuperscript{1,6}\thanks{Corresponding author:
    \href{mailto:dejingdou@gmail.com}{\texttt{dejingdou@gmail.com}}}
  \\[3pt]
  \normalfont\large
  \textsuperscript{1}College of Computer Science and Artificial Intelligence, Fudan University
  \\
  \textsuperscript{2}School of Electrical and Information Engineering, Tianjin University
  \\
  \textsuperscript{3}King Abdullah University of Science and Technology (KAUST), Saudi Arabia
  \\
  \textsuperscript{4}School of Software Technology, Zhejiang University
  \\
  \textsuperscript{5}School of Transnational Law, Peking University\quad
  \textsuperscript{6}BEDI Cloud
}

\begin{document}
\AddToShipoutPictureFG*{%
  \AtPageLowerLeft{%
    \put(0,35){%
      \makebox[\paperwidth][c]{%
        \normalfont\fontsize{9}{11}\selectfont\itshape
        Accepted to EMNLP 2026 (Main Conference).%
      }%
    }%
  }%
}
\maketitle
\begin{abstract}
Large Language Models (LLMs) inevitably internalize substantial amounts of sensitive or private information during pre-training, while LLM unlearning aims to selectively erase specific knowledge to prevent privacy leakage with minimal loss of model utility. However, existing methods struggle to balance forget quality with utility, and typically incur substantial computational costs due to parameter fine-tuning. To address this, we propose \textbf{Nullify}, a training-free, non-destructive activation steering method for LLM unlearning. Nullify employs steering vectors during inference to redirect privacy-related activations away from their memorized answers, while satisfying a null-space constraint that leaves retained-query activations essentially unaffected to maintain utility. Evaluations on TOFU and MUSE show that Nullify matches or surpasses established baselines in forget quality while achieving \textbf{near-lossless} preservation of model utility. By avoiding weight updates entirely, Nullify serves as an efficient, plug-and-play inference-time intervention framework.
\end{abstract}

\section{Introduction}

The pre-training of Large Language Models (LLMs) on massive web corpora inevitably leads to the implicit memorization of sensitive, copyrighted, or harmful data \cite{carlini2021extracting}. Legal frameworks such as the GDPR recognize a right to erasure under certain conditions \cite{bourtoule2021machine}, while emerging AI regulations such as the EU AI Act increase compliance pressure around data governance and model lifecycle management \cite{eu_ai_act_2024}. Because retraining a foundation model from scratch for every such request is economically infeasible, LLM unlearning---the targeted, efficient elimination of specific memorized knowledge---has become a cornerstone of trustworthy AI \cite{yao2023large}.

Yet the dominant unlearning paradigm is fundamentally \emph{destructive}. Gradient-based methods (GA \cite{jang2022knowledge}, GradDiff \cite{liu2022few}) and preference-optimization methods (NPO \cite{zhang2024negative}, SimNPO \cite{fan2025simplicity}) all unlearn by backpropagating into and reshaping the model's weights. Because the knowledge to forget and the capabilities to retain are encoded in the \emph{same} shared parameters, every update that suppresses a target memory unavoidably perturbs the retained representations: forgetting and utility become locked in direct opposition, tradable only at the cost of capability degradation, over-refusal, and hallucination \cite{fan2025simplicity}. Representation-level methods (RMU \cite{li2024wmdp}, LUNAR \cite{shen2025lunar}) soften but do not escape this: they either still update the base model's weights or train auxiliary parameters, and their retention behavior is enforced by a soft loss term rather than guaranteed by construction. As Table~\ref{tab:method_comparison} summarizes, no prior method simultaneously avoids weight modification, eliminates optimization, and provides a principled utility-preservation guarantee.

\begin{table}[t]
  \centering
  \small
  \setlength{\tabcolsep}{5pt}
  \begin{tabular}{l | c c c c c}
    \toprule
    Method & TF & NWM & GC & HFQ & UP \\
    \midrule
    GradDiff         & \xmark & \xmark & \xmark & \xmark & \xmark \\
    SimNPO           & \xmark & \xmark & \xmark & \cmark & \xmark \\
    LUNAR            & \xmark & \cmark & \xmark & \xmark & \xmark \\
    \textbf{Nullify} & \cmark & \cmark & \cmark & \cmark & \cmark \\
    \bottomrule
  \end{tabular}
  \caption{Attribute comparison of representative LLM unlearning methods.
  TF: Training-Free (no backpropagation);
  NWM: No (base) Weight Modification;
  GC: Geometric Constraint isolating retained knowledge;
  HFQ: High Forget Quality; UP: Utility Preservation.
}
  \label{tab:method_comparison}
\end{table}

Our key observation is that this opposition is \emph{not} intrinsic to unlearning---it is an artifact of operating in parameter space. If the intervention is instead applied in the \emph{activation space} and \emph{geometrically constrained} to a subspace orthogonal to the retained-knowledge representations, suppressing a target memory and preserving utility cease to compete: the former is delivered by the intervention, the latter holds because the intervention, by construction, has no effect on retained queries. Building on this insight, we propose \textbf{Nullify}, a training-free unlearning framework with a closed-form, null-space-constrained utility-preservation construction. Nullify keeps the base LLM \emph{entirely frozen} and introduces no trainable parameters: a single per-layer linear operator, solved in closed form, steers forget queries off their memorized-answer trajectory while a null-space projector confines every correction to the subspace orthogonal to the retained-knowledge activations. As Figure~\ref{fig:pareto} previews on TOFU \texttt{forget10}, Nullify achieves the best forget quality at the lowest wall-clock cost we measure, while every weight-space baseline cedes ground on forgetting, utility, or compute.

\begin{figure}[t]
  \centering
  \begin{subfigure}{0.49\linewidth}
    \centering
    \includegraphics[width=\linewidth]{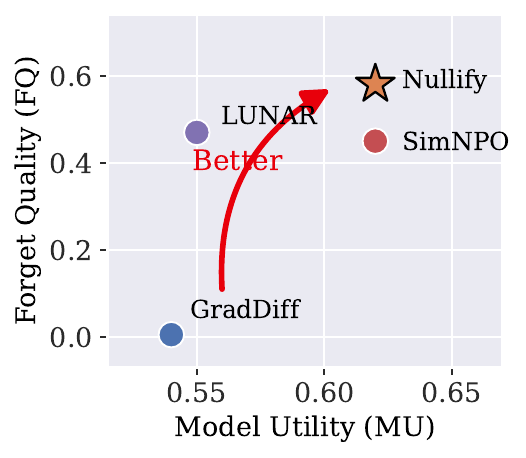}
    \caption{Forget quality vs.\ model utility.}
    \label{fig:pareto-mu}
  \end{subfigure}
  \hfill
  \begin{subfigure}{0.49\linewidth}
    \centering
    \includegraphics[width=\linewidth]{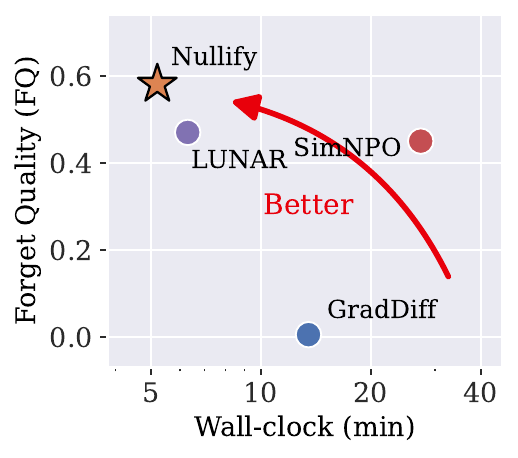}
    \caption{Forget quality vs.\ wall-clock cost.}
    \label{fig:pareto-cost}
  \end{subfigure}
  \caption{Pareto view of LLM unlearning on TOFU \texttt{forget10}
  (Llama-2-7B-chat-hf). Nullify (orange star) dominates all baselines,
  achieving the highest forget quality at the highest model utility~(a)
  and the lowest wall-clock cost~(b); the red arrow marks
  the preferred direction. Wall-clock in (b) refers to the cost benchmark
  of Table~\ref{tab:cost}.}
  \label{fig:pareto}
\end{figure}

Our main contributions are threefold:
\begin{itemize}
    \item \textbf{A non-destructive, guarantee-backed paradigm.} We recast unlearning as closed-form, null-space-constrained activation steering whose correction vanishes on retain queries by construction---to our knowledge the first unlearning method that is at once training-free, weight-preserving, and backed by a principled utility-preservation construction.
    \item \textbf{A fully closed-form, annotation-free pipeline.} The entire per-layer operator is solved in closed form from one forward pass---zero trainable parameters, no backpropagation---with the forget direction derived purely from IDK-contrastive activations and the steering strength being a runtime scalar.
    \item \textbf{State-of-the-art on the forget--utility--cost trilemma.} On TOFU \cite{maini2024tofu} and MUSE \cite{shi2024muse}, Nullify leads baselines in forget quality while keeping utility near-lossless at a small fraction of their compute---the unique Pareto-dominant point among unlearning methods.
\end{itemize}

\section{Related Work}

\subsection{LLM Unlearning}
LLM unlearning aims to remove the influence of specific data from a trained model without full retraining. Exact unlearning---retraining from scratch on the retain set---is the gold standard, but its computational cost and reliance on the complete training corpus make it impractical at LLM scale \cite{kurmanji2024towards, thudi2022unrolling}. The field has therefore converged on \emph{approximate} unlearning, and its dominant family induces forgetting by directly modifying the model's weights. Gradient-based methods maximize the loss on the forget set or contrast it with a retain gradient (GA \cite{jang2022knowledge}, GradDiff \cite{liu2022few}); preference-optimization methods recast forgetting as alignment away from the memorized answer (DPO \cite{rafailov2023direct}, NPO \cite{zhang2024negative}, SimNPO \cite{fan2025simplicity}); and task-vector methods subtract a task-specific weight increment in parameter space \cite{ilharco2022editing}.

Because forgetting and retention are entangled in the same parameters, this paradigm is inherently destructive: it is well documented to erode general utility, and it further incurs substantial computational cost and acute hyperparameter sensitivity \cite{fan2025simplicity, wang2024unlearning}. To standardize evaluation, diverse benchmarks have been proposed---fictitious QA (TOFU \cite{maini2024tofu}), copyrighted corpora (MUSE \cite{shi2024muse}), and hazardous knowledge (WMDP \cite{li2024wmdp})---making the balance between forgetting and utility the field's central criterion. 

\subsection{Representation-Level and Inference-Time Unlearning}
To avoid the collateral damage of destructive weight updates, a recent line of work moves the unlearning intervention from parameters to internal representations, building on activation steering---an inference-time technique that adds a vector to a model's hidden states without any gradient update \cite{rimsky2024steering}, grounded in the linear representation hypothesis \cite{zou2023representation, park2023linear}. Feature-intervention methods redirect the hidden states of forget queries toward random or fixed directions, with RMU \cite{li2024wmdp} as a representative instance; activation-redirection methods such as LUNAR \cite{shen2025lunar} map forget activations onto a refusal-style reference through a learned projection \cite{dang2025effects}. A parallel, training-light direction enforces forgetting at the decoding stage rather than in the representation space, e.g., GUARD \cite{deng2025guard} filters forbidden tokens during generation.

These approaches improve stability over full fine-tuning, yet none fully escapes the original tension: they either still fine-tune  the base model or require training auxiliary parameters, and in both cases retention is enforced through a soft loss term and inherited as a by-product of optimization---with no mechanism that decouples retention from forgetting by construction.


\section{Methodology}

We present \textbf{Nullify}, a training-free activation-steering framework for LLM unlearning. The setting is the standard one~\cite{maini2024tofu,shi2024muse}: an LLM $\Phi$ has been fine-tuned on a corpus containing a designated \emph{forget set} $\mathcal{D}_f=\{(q_i,a_i)\}_{i=1}^{N_f}$ that must be erased, while a \emph{keep set} $\mathcal{D}_k$---the union of the retain split and held-out general-utility prompts---must remain unaffected. Let $\ell\in\{1,\dots,L\}$ index the decoder layers of $\Phi$, $\mathbf{h}^{(\ell)}\in\mathbb{R}^{D}$ a hidden state, and $\mathcal{L}\subseteq\{1,\dots,L\}$ a fixed subset of steered layers.

Rather than fine-tuning $\Phi$---which entangles forgetting with general capability through shared parameters and is the source of the utility collapse and compute cost of gradient-based unlearning---Nullify keeps every weight frozen and intervenes only on hidden states at inference. At each steered layer it adds a per-prompt linear correction $\alpha\,\mathbf{M}^{(\ell)}\mathbf{h}_{\star}^{(\ell)}$ to the residual stream, with the operator $\mathbf{M}^{(\ell)}\in\mathbb{R}^{D\times D}$ obtained in closed form from a single forward pass per prompt over $\mathcal{D}_f\cup\mathcal{D}_k$, no trainable parameters, and no backpropagation. Sec.~\ref{sec:nullify_intervention} fixes the form of the intervention and the contract its operator must satisfy; Sec.~\ref{sec:nullify_direction} derives the unlearning-specific target direction the contract pushes forget activations toward; Sec.~\ref{sec:nullify_pipeline} solves for the two factors of $\mathbf{M}^{(\ell)}$ in closed form; and Sec.~\ref{sec:nullify_algo} states the resulting algorithm and its cost.

\subsection{Intervention Form and Design Contract}
\label{sec:nullify_intervention}

Nullify intervenes during the prompt-encoding pass, before any answer token is produced. For each prompt of length $T$, designate the \emph{anchor} as the hidden state at the last user-prompt token,
\begin{equation}
\label{eq:nullify_anchor}
\mathbf{h}_{\star}^{(\ell)}\;:=\;\mathbf{h}_{T}^{(\ell)},
\end{equation}
the position at which the residual stream has integrated the entire query~\cite{arditi2024refusal}. From this single anchor Nullify computes one correction and broadcasts it across every position of the prompt:
\begin{equation}
\label{eq:nullify_intervention}
\mathbf{h}_{t}^{(\ell)}\leftarrow\mathbf{h}_{t}^{(\ell)}+\alpha\,\mathbf{M}^{(\ell)}\,\mathbf{h}_{\star}^{(\ell)},
\, t=1,\dots,T,\;\ell\in\mathcal{L},
\end{equation}
where $\alpha\in\mathbb{R}$ is a scalar strength. Two consequences of this anchor-broadcast scheme are worth noting. First, because the correction is determined by $\mathbf{h}_{\star}^{(\ell)}$, it is \emph{prompt-adaptive}: different queries receive different corrections, even though $\mathbf{M}^{(\ell)}$ itself is fixed. Second, computing the correction once and reusing it across token positions reduces the per-prompt overhead to a single $D\times D$ matrix--vector product per steered layer; the steered hidden states then populate the cached keys and values, so autoregressive generation runs without any per-token hook.

What should $\mathbf{M}^{(\ell)}\mathbf{h}_{\star}^{(\ell)}$ become on a given prompt? Unlearning gives a sharp answer: \emph{nothing} on a keep query (otherwise utility leaks), and a target displacement $\mathbf{r}^{(\ell)}$ on a forget query---the displacement that takes the activation off its memorized-answer trajectory; we define it in Sec.~\ref{sec:nullify_direction}. We state this as the operator's design contract:
\begin{equation}
\label{eq:nullify_requirements}
\mathbf{M}^{(\ell)}\,\mathbf{h}_{\star}^{(\ell)}\;\approx\;
\begin{cases}
\mathbf{0}, & q \text{ is a keep query},\\[2pt]
\mathbf{r}^{(\ell)}, & q \text{ is a forget query}.
\end{cases}
\end{equation}
The two cases impose unrelated constraints on the same operator. To solve them independently we factorize
\begin{equation}
\label{eq:nullify_factorization}
\mathbf{M}^{(\ell)}\;=\;\tilde{\boldsymbol{\Delta}}^{(\ell)}\,\hat{\mathbf{P}}^{(\ell)},
\end{equation}
assigning the \emph{keep} case to the right factor $\hat{\mathbf{P}}^{(\ell)}$ and the \emph{forget} case to the left factor $\tilde{\boldsymbol{\Delta}}^{(\ell)}$: $\hat{\mathbf{P}}^{(\ell)}$ is built to send every keep anchor to zero, so that the full operator $\mathbf{M}^{(\ell)}\mathbf{h}_{\star}^{(\ell)}$ vanishes on keep queries regardless of the choice of $\tilde{\boldsymbol{\Delta}}^{(\ell)}$; the latter is then chosen to map the projected forget anchors onto $\mathbf{r}^{(\ell)}$. Sec.~\ref{sec:nullify_direction} supplies $\mathbf{r}^{(\ell)}$ and Sec.~\ref{sec:nullify_pipeline} constructs both factors.

\subsection{The Forget Direction: IDK-Contrastive Activations}
\label{sec:nullify_direction}

The contract in Eq.~\eqref{eq:nullify_requirements} requires a per-layer target $\mathbf{r}^{(\ell)}\in\mathbb{R}^{D}$ that captures, in activation space, the \emph{signature of memorization}---the systematic shift a fine-tuned LLM accumulates when it has internalized a factual association and is committing to reproduce it. Unlearning a fact then reduces to canceling this shift in the corresponding activations.

Constructing $\mathbf{r}^{(\ell)}$ is a key challenge specific to the unlearning setting. Forget queries are ordinary factual questions, surface-indistinguishable from retained ones, with no behavioral annotation available, so the target direction cannot be read off labeled samples. We therefore \emph{construct} the direction contrastively, using only the forget set's own $(q,a)$ pairs. For each forget query $q_i$ with memorized answer $a_i$, we read the frozen model's layer-$\ell$ hidden state at the last token of two prompts---one conditioning on the memorized continuation, one on abstention:
\begin{equation}
\label{eq:nullify_extract}
\left\{
\begin{array}{l}
\mathbf{h}_{i}^{(\ell,\mathrm{ans})}=\Phi^{(\ell)}\!\bigl(q_i\oplus a_i\bigr),\\[3pt]
\mathbf{h}_{i}^{(\ell,\mathrm{idk})}=\Phi^{(\ell)}\!\bigl(q_i\oplus\texttt{IDK}\bigr),
\end{array}
\right.
\end{equation}
where $\oplus$ denotes prompt concatenation, \texttt{IDK} is a fixed ``I don't know''-style template, and $\Phi^{(\ell)}(\cdot)\in\mathbb{R}^{D}$ returns the layer-$\ell$ last-token hidden state. Averaging over the forget set yields the two population modes
\begin{equation}
\label{eq:nullify_means}
\left\{
\begin{array}{rl}
\boldsymbol{\mu}_{\mathrm{ans}}^{(\ell)} &=\dfrac{1}{N_f}\displaystyle\sum_{i=1}^{N_f}\mathbf{h}_{i}^{(\ell,\mathrm{ans})},\\[10pt]
\boldsymbol{\mu}_{\mathrm{idk}}^{(\ell)} &=\dfrac{1}{N_f}\displaystyle\sum_{i=1}^{N_f}\mathbf{h}_{i}^{(\ell,\mathrm{idk})},
\end{array}
\right.
\end{equation}
and the \emph{forget direction} is their residual:
\begin{equation}
\label{eq:nullify_forget_direction}
\mathbf{r}^{(\ell)}\;=\;\boldsymbol{\mu}_{\mathrm{ans}}^{(\ell)}-\boldsymbol{\mu}_{\mathrm{idk}}^{(\ell)}\;\in\;\mathbb{R}^{D}.
\end{equation}
Geometrically, $\mathbf{r}^{(\ell)}$ is the average displacement from abstention to commitment; with the sign convention $\alpha<0$ used throughout Eq.~\eqref{eq:nullify_intervention}, adding $\alpha\mathbf{r}^{(\ell)}$ to the residual stream slides a forget activation off the memorized-answer mode $\boldsymbol{\mu}_{\mathrm{ans}}^{(\ell)}$ toward the abstention mode $\boldsymbol{\mu}_{\mathrm{idk}}^{(\ell)}$.


\subsection{Solving the Two Factors in Closed Form}
\label{sec:nullify_pipeline}

We now construct the two factors of Eq.~\eqref{eq:nullify_factorization} so that the contract in Eq.~\eqref{eq:nullify_requirements} is met. For notational clarity we suppress the layer superscript $(\ell)$ throughout this subsection; all quantities are understood as per-layer.

\paragraph{Right factor $\hat{\mathbf{P}}$: annihilating the keep subspace.}
Let $\mathbf{H}_{k}\in\mathbb{R}^{D\times N_k}$ stack the anchor activations of every prompt in $\mathcal{D}_k$. The keep case of Eq.~\eqref{eq:nullify_requirements} demands $\mathbf{M}\mathbf{H}_{k}=\mathbf{0}$. Substituting the factorization, this reduces to $\tilde{\boldsymbol{\Delta}}\hat{\mathbf{P}}\mathbf{H}_{k}=\mathbf{0}$, which holds for \emph{any} choice of $\tilde{\boldsymbol{\Delta}}$ whenever every column of $\mathbf{H}_{k}$ lies in the kernel of $\hat{\mathbf{P}}$. We therefore take $\hat{\mathbf{P}}$ to be the orthogonal projector onto the left null space of $\mathbf{H}_{k}$.

\begin{definition}[Left null space]
\label{def:nullspace}
For $\mathbf{H}\in\mathbb{R}^{D\times N}$, $\mathrm{Null}_{L}(\mathbf{H})=\{\mathbf{x}\in\mathbb{R}^{D}:\mathbf{x}^{\top}\mathbf{H}=\mathbf{0}\}$.
\end{definition}

Computing this projector by a direct SVD of $\mathbf{H}_{k}$ is wasteful when $N_k\gg D$; the following standard equivalence~\cite{wang2021training,fang2025alphaedit,sheng2025alphasteer} reduces it to the $D\times D$ Gram matrix.

\begin{lemma}[Null-space equivalence]
\label{lem:null_equiv}
$\mathrm{Null}_{L}(\mathbf{H})=\mathrm{Null}_{L}(\mathbf{H}\mathbf{H}^{\top})$.
\end{lemma}

By Lemma~\ref{lem:null_equiv} (proof in Appendix~\ref{app:lemma_proof}),
we accordingly compute an SVD $\mathbf{H}_{k}\mathbf{H}_{k}^{\top}=\mathbf{V}\boldsymbol{\Sigma}\mathbf{V}^{\top}$, collect the eigenvectors of the $k=\lfloor pD\rfloor$ smallest eigenvalues into $\hat{\mathbf{V}}\in\mathbb{R}^{D\times k}$, and set
\begin{equation}
\label{eq:nullify_projector}
\hat{\mathbf{P}}\;=\;\hat{\mathbf{V}}\hat{\mathbf{V}}^{\top}\;\in\;\mathbb{R}^{D\times D},
\end{equation}
where the \emph{null-space ratio} $p\in(0,1)$ is a structural hyperparameter (sensitivity reported in Sec.~\ref{sec:sensitivity}). $\hat{\mathbf{P}}$ is a symmetric idempotent---$\hat{\mathbf{P}}^{\top}=\hat{\mathbf{P}}$ and $\hat{\mathbf{P}}^{2}=\hat{\mathbf{P}}$---and $\hat{\mathbf{P}}\mathbf{H}_{k}=\mathbf{0}$ holds down to the smallest retained eigenvalue of the Gram. The keep case of Eq.~\eqref{eq:nullify_requirements} is thus discharged \emph{independently} of $\tilde{\boldsymbol{\Delta}}$.

\paragraph{Left factor $\tilde{\boldsymbol{\Delta}}$: a closed-form forget mapping.}
With $\hat{\mathbf{P}}$ fixed, the forget case of Eq.~\eqref{eq:nullify_requirements} becomes: choose $\tilde{\boldsymbol{\Delta}}$ so that $\tilde{\boldsymbol{\Delta}}\hat{\mathbf{P}}\mathbf{H}_{f}\approx\mathbf{R}$, where $\mathbf{H}_{f}\in\mathbb{R}^{D\times N_f}$ stacks the anchor activations of the raw forget queries $\{q_i\}$ and $\mathbf{R}=\mathbf{r}\,\mathbf{1}_{N_f}^{\top}$ stacks $N_f$ copies of the forget direction. We solve it as a Frobenius-regularized least-squares problem,
\begin{equation}
\label{eq:nullify_objective}
\min_{\tilde{\boldsymbol{\Delta}}}\;
\big\lVert\tilde{\boldsymbol{\Delta}}\hat{\mathbf{P}}\mathbf{H}_{f}-\mathbf{R}\big\rVert_{F}^{2}
+\lambda\big\lVert\tilde{\boldsymbol{\Delta}}\hat{\mathbf{P}}\big\rVert_{F}^{2},
\end{equation}
whose closed-form minimizer, after applying the symmetric-idempotent identities of $\hat{\mathbf{P}}$, is
\begin{equation}
\label{eq:nullify_ridge}
\tilde{\boldsymbol{\Delta}}^{\star}=\mathbf{R}\,\mathbf{H}_{f}^{\top}\hat{\mathbf{P}}\,
\bigl(\hat{\mathbf{P}}\mathbf{H}_{f}\mathbf{H}_{f}^{\top}\hat{\mathbf{P}}+\lambda\hat{\mathbf{P}}\bigr)^{\!+},
\end{equation}
with $(\cdot)^{+}$ the Moore--Penrose pseudoinverse and $\lambda>0$ the ridge coefficient; the full derivation, including the handling 
of the rank deficiency of $\hat{\mathbf{P}}$, is given in 
Appendix~\ref{app:closed_form}.

\paragraph{Assembly.}
Restoring the layer index, the per-layer steering operator at layer $\ell$ is
\begin{equation}
\label{eq:nullify_operator}
\mathbf{M}^{(\ell)}\;=\;\tilde{\boldsymbol{\Delta}}^{\star(\ell)}\hat{\mathbf{P}}^{(\ell)},
\end{equation}
applied at inference through Eq.~\eqref{eq:nullify_intervention}. By construction, $\hat{\mathbf{P}}^{(\ell)}$ annihilates every column of $\mathbf{H}_{k}^{(\ell)}$ down to the smallest retained eigenvalue of the Gram, while $\tilde{\boldsymbol{\Delta}}^{\star(\ell)}$ delivers the ridge-regularized best fit of $\mathbf{r}^{(\ell)}$ on the projected forget anchors---the two cases of the contract in Eq.~\eqref{eq:nullify_requirements}, separately and tightly.

\subsection{Algorithm and Complexity}
\label{sec:nullify_algo}

Algorithm~\ref{alg:nullify} summarizes the offline construction: a single forward pass per prompt over $\mathcal{D}_k\cup\mathcal{D}_f$ (plus the IDK-contrastive auxiliary prompts), one $D\times D$ Gram SVD, and one pseudoinverse---per steered layer.


\paragraph{Cost.}
Each layer costs $O(N_k D^2)$ to form the keep Gram and $O(D^3)$ for the subsequent SVD and pseudoinverse; the closed-form left-factor adds $O(N_f D^2)$. Total offline cost per layer is $O\bigl((N_k+N_f)D^2 + D^3\bigr)$. \emph{At inference}, the steered hidden states populate the cached keys and values, so generation reuses the steered KV cache without any per-token hook; the per-prompt overhead is a single $D\times D$ matrix--vector product per steered layer, asymptotically negligible against the base forward pass. Nullify adds zero trainable parameters and never invokes a backward pass; it exposes four choices---the steered-layer set $\mathcal{L}$ and three scalars $(p,\lambda,\alpha)$---all subjected to sensitivity analysis in Sec.~\ref{sec:sensitivity}.

\begin{algorithm}[t]
\caption{Nullify (offline operator construction).}
\label{alg:nullify}
\begin{algorithmic}[1]
\REQUIRE frozen LLM $\Phi$; forget set $\mathcal{D}_f$; keep set $\mathcal{D}_k$; steered layers $\mathcal{L}$; hyperparameters $(p,\lambda)$.
\FOR{$\ell\in\mathcal{L}$}
  \STATE $\mathbf{H}_{k}^{(\ell)}\!\leftarrow$ anchor activations of all prompts in $\mathcal{D}_k$
  \STATE $\mathbf{H}_{f}^{(\ell)}\!\leftarrow$ anchor activations of forget queries $\{q_i\}$ from $\mathcal{D}_f$
  \STATE compute $\{\mathbf{h}_{i}^{(\ell,\mathrm{ans})},\mathbf{h}_{i}^{(\ell,\mathrm{idk})}\}$ via Eq.~\eqref{eq:nullify_extract}
  \STATE $\mathbf{r}^{(\ell)}\!\leftarrow\boldsymbol{\mu}_{\mathrm{ans}}^{(\ell)}-\boldsymbol{\mu}_{\mathrm{idk}}^{(\ell)}$
  \STATE $\hat{\mathbf{P}}^{(\ell)}\!\leftarrow$ Eq.~\eqref{eq:nullify_projector} from $\mathbf{H}_{k}^{(\ell)}{\mathbf{H}_{k}^{(\ell)}}^{\!\top}$
  \STATE $\tilde{\boldsymbol{\Delta}}^{\star(\ell)}\!\leftarrow$ closed-form solution of Eq.~\eqref{eq:nullify_ridge}
  \STATE $\mathbf{M}^{(\ell)}\!\leftarrow\tilde{\boldsymbol{\Delta}}^{\star(\ell)}\hat{\mathbf{P}}^{(\ell)}$
\ENDFOR
\ENSURE per-layer operators $\{\mathbf{M}^{(\ell)}\}_{\ell\in\mathcal{L}}$, applied at inference via Eq.~\eqref{eq:nullify_intervention}.
\end{algorithmic}
\end{algorithm}

\section{Experiments}

\subsection{Setup}
\label{sec:setup}

\paragraph{Benchmarks.}
We evaluate Nullify on two standard LLM unlearning benchmarks. \textbf{TOFU}~\cite{maini2024tofu} consists of 4{,}000 question--answer pairs about 200 fictitious authors; because the target knowledge is entirely synthetic and introduced by supervised fine-tuning (SFT), unlearning can be assessed against a model retrained from scratch on the retain split. We report the \texttt{forget05} and \texttt{forget10} settings, which erase 5\% and 10\% of the data. \textbf{MUSE}~\cite{shi2024muse} simulates real-world privacy and copyright removal; we evaluate on its News corpus (MUSE-News).

\paragraph{Base model.}
For the main TOFU experiments, all methods operate on \texttt{Llama-2-7B-chat-hf} fine-tuned on the full TOFU corpus; the retrained reference is obtained by fine-tuning the same model on the retain split only. Appendix~\ref{app:cross-backbone} reports a cross-family replication on Qwen3-8B~\cite{yang2025qwen3} under \texttt{forget10}. MUSE-News follows the benchmark's standard setup.

\paragraph{Baselines.}
Besides the \emph{Retrain} gold reference~\cite{bourtoule2021machine} and the un-unlearned \emph{Original} model, we compare against four families: \emph{gradient-based} (GA~\cite{jang2022knowledge}, GradDiff~\cite{liu2022few}); \emph{weight editing} (Task Vector~\cite{ilharco2022editing}); \emph{preference optimization} (IDK~\cite{eldan2023s}, NPO~\cite{zhang2024negative}, and its reference-free variant SimNPO~\cite{fan2025simplicity}); and \emph{representation intervention} (LUNAR~\cite{shen2025lunar}).

\paragraph{Metrics.}
On TOFU, the two summary metrics are \textbf{Forget Quality (FQ)}---the Kolmogorov--Smirnov $p$-value between the truth-ratio distributions of the unlearned and retrained models, so that statistical indistinguishability from Retrain yields $\mathrm{FQ}\!\to\!1$ and \textbf{Model Utility (MU)}, the harmonic mean of performance over the non-target Real Authors, World Facts, and Retain sets, where any single capability drop sharply penalizes the aggregate. For completeness we also report per-set ROUGE-L (RL), answer probability (P), and truth ratio (TR). On MUSE-News we report \textbf{KM-D$_\text{f}$} (forget-set knowledge memorization), \textbf{KM-D$_\text{r}$} (retain-set knowledge), and \textbf{PrivLeak}, a membership-inference-based privacy score whose ideal value is $0$.

\paragraph{Implementation.}
Nullify keeps the base model frozen and adds no trainable parameters. Following the localization of factual recall to mid-upper transformer layers~\cite{geva2021transformer,meng2022locating} and the layer-selection convention of AlphaSteer~\cite{sheng2025alphasteer}, we steer a fixed set of ten mid-upper decoder layers ($\mathcal{L}=\{8\text{--}14,16,18,19\}$). Unless noted otherwise, we use null-space ratio $p\!=\!0.75$, ridge coefficient $\lambda\!=\!10$, and steering strength $\alpha\!=\!-2.0$. Experiments conducted in this work use greedy decoding on a single NVIDIA H800 80\,GB GPU. For direct comparability, we follow the same benchmark setting and evaluation protocol as~\citet{fan2025simplicity}, including the backbone, forget/retain splits, and metric definitions. Values in $\ddagger$-marked rows of Tables~\ref{tab:tofu} and~\ref{tab:muse} are taken from~\citet{fan2025simplicity}; unmarked methods are evaluated in this work under this matched protocol.

\begin{table*}[!t]
  \centering\small
  \setlength{\tabcolsep}{4pt}
  \definecolor{lightgray}{gray}{0.92}
  \begin{tabular}{l | ccc c | ccc ccc ccc c}
    \toprule
    \multirow{3}{*}{\textbf{Method}}
      & \multicolumn{4}{c|}{\textbf{Unlearning Efficacy}}
      & \multicolumn{10}{c}{\textbf{Utility Preservation}} \\
    \cmidrule(lr){2-5} \cmidrule(lr){6-15}
      & \multicolumn{3}{c}{Forget Set} & \multirow{2}{*}{FQ $\uparrow$}
      & \multicolumn{3}{c}{Real Authors $\uparrow$}
      & \multicolumn{3}{c}{World Facts $\uparrow$}
      & \multicolumn{3}{c}{Retain Set $\uparrow$}
      & \multirow{2}{*}{MU $\uparrow$} \\
    \cmidrule(lr){2-4} \cmidrule(lr){6-8} \cmidrule(lr){9-11} \cmidrule(lr){12-14}
      & 1-RL & 1-P & TR &
      & RL & P & TR & RL & P & TR & RL & P & TR & \\
    \midrule
    \rowcolor{lightgray}\multicolumn{15}{c}{\textit{Forget05}} \\
    \midrule
    Original$^\ddagger$   & 0.04 & 0.01 & 0.49 & 0.00 & 0.93 & 0.44 & 0.58 & 0.91 & 0.43 & 0.55 & 0.98 & 0.99 & 0.48 & 0.62 \\
    Retrain$^{\dagger,\ddagger}$ & 0.61 & 0.85 & 0.66 & 1.00 & 0.92 & 0.44 & 0.57 & 0.90 & 0.43 & 0.54 & 0.97 & 0.99 & 0.48 & 0.62 \\
    \midrule
    GA$^\ddagger$         & 1.00 & 1.00 & 0.66 & $<$0.01 & 0.00 & 0.20 & 0.40 & 0.00 & 0.30 & 0.28 & 0.00 & 0.00 & 0.15 & 0.00 \\
    GradDiff$^\ddagger$   & 1.00 & 1.00 & 0.60 & $<$0.01 & 0.59 & 0.59 & 0.81 & 0.88 & 0.46 & 0.59 & 0.42 & 0.49 & 0.48 & 0.56 \\
    IDK$^\ddagger$        & 0.98 & 0.40 & 0.55 & $<$0.01 & 0.65 & 0.48 & 0.63 & 0.82 & 0.44 & 0.55 & 0.55 & 0.86 & 0.43 & 0.57 \\
    NPO$^\ddagger$        & 0.73 & 0.94 & 0.67 & 0.79 & 0.91 & 0.50 & 0.62 & 0.90 & 0.50 & 0.61 & 0.47 & 0.51 & 0.44 & 0.57 \\
    SimNPO$^\ddagger$     & 0.74 & 0.97 & 0.69 & \textbf{0.99} & 0.90 & 0.50 & 0.64 & 0.90 & 0.48 & 0.60 & 0.54 & 0.56 & 0.44 & 0.58 \\
    LUNAR      & 0.70 & 0.69 & 0.69 & 0.87 & 0.71 & 0.49 & 0.63 & 0.80 & 0.44 & 0.57 & 0.69 & 0.82 & 0.43 & 0.59 \\
    \textbf{Nullify} & 0.62 & 0.88 & 0.69 & 0.96 & 0.92 & 0.47 & 0.61 & 0.88 & 0.44 & 0.58 & 0.95 & 0.86 & 0.44 & \textbf{0.62} \\
    \midrule
    \rowcolor{lightgray}\multicolumn{15}{c}{\textit{Forget10}} \\
    \midrule
    Original$^\ddagger$   & 0.03 & 0.01 & 0.48 & 0.00 & 0.93 & 0.44 & 0.58 & 0.91 & 0.43 & 0.55 & 0.98 & 0.99 & 0.48 & 0.62 \\
    Retrain$^{\dagger,\ddagger}$ & 0.61 & 0.84 & 0.67 & 1.00 & 0.93 & 0.45 & 0.59 & 0.91 & 0.42 & 0.54 & 0.98 & 0.99 & 0.47 & 0.62 \\
    \midrule
    GA$^\ddagger$         & 1.00 & 1.00 & 0.70 & $<$0.01 & 0.00 & 0.28 & 0.37 & 0.00 & 0.29 & 0.31 & 0.00 & 0.00 & 0.11 & 0.00 \\
    GradDiff$^\ddagger$   & 1.00 & 1.00 & 0.67 & $<$0.01 & 0.44 & 0.49 & 0.67 & 0.89 & 0.48 & 0.58 & 0.48 & 0.60 & 0.46 & 0.54 \\
    IDK$^\ddagger$        & 0.98 & 0.37 & 0.54 & $<$0.01 & 0.46 & 0.45 & 0.59 & 0.84 & 0.43 & 0.55 & 0.56 & 0.88 & 0.44 & 0.54 \\
    NPO$^\ddagger$        & 0.78 & 0.90 & 0.70 & 0.29 & 0.91 & 0.52 & 0.66 & 0.85 & 0.48 & 0.61 & 0.44 & 0.46 & 0.39 & 0.55 \\
    SimNPO$^\ddagger$     & 0.78 & 0.91 & 0.71 & 0.45 & 0.90 & 0.54 & 0.70 & 0.88 & 0.50 & 0.64 & 0.54 & 0.76 & 0.47 & \textbf{0.62} \\
    LUNAR      & 0.83 & 0.84 & 0.67 & 0.47 & 0.60 & 0.52 & 0.67 & 0.72 & 0.43 & 0.59 & 0.53 & 0.66 & 0.43 & 0.55 \\
    \textbf{Nullify} & 0.59 & 0.92 & 0.71 & \textbf{0.58} & 0.92 & 0.47 & 0.61 & 0.90 & 0.42 & 0.55 & 0.95 & 0.88 & 0.45 & \textbf{0.62} \\
    \bottomrule
  \end{tabular}
  \caption{Results on TOFU~\cite{maini2024tofu} (\texttt{forget05}/\texttt{forget10}, Llama-2-7B-chat-hf). \textbf{Unlearning Efficacy}: forget-set 1$-$ROUGE-L, 1$-$Prob., Truth Ratio, and the summary metric FQ. \textbf{Utility Preservation}: RL/P/TR on Real Authors, World Facts, and the Retain set, summarized by MU. $^\dagger$Retrain is the gold reference; the forget-set columns are best when \emph{closest to Retrain}, not when largest. Rows marked with $^\ddagger$ are taken from~\citet{fan2025simplicity}. \textbf{Bold}: best FQ / MU among unlearning methods.}
  \label{tab:tofu}
\end{table*}

\subsection{Main Results}

Tables~\ref{tab:tofu} and~\ref{tab:muse} report quantitative results on TOFU and MUSE-News. Across both benchmarks Nullify delivers the strongest forget--utility balance: it suppresses the target knowledge while leaving general capability essentially intact, avoiding the utility collapse that the gradient- and preference-based baselines incur. We detail the per-benchmark results below, contrasting Nullify against the baselines on both forget quality and utility.

\subsubsection{Results on TOFU}

The gradient-based baselines expose the cost of destructive updates. Gradient Ascent maximizes forget-set loss but destroys the model: its MU falls to $0.00$ under both forget ratios, a complete loss of linguistic competence. GradDiff and IDK avoid total collapse yet still leave MU at $0.54$--$0.57$, well below the $0.62$ of the Original and Retrain models.

Nullify is the only method whose utility is indistinguishable from the references: MU holds at $0.62$ on both \texttt{forget05} and \texttt{forget10}, matching Retrain exactly. On \texttt{forget05} it reaches an FQ of $0.96$, within $0.03$ of the strongest baseline SimNPO ($0.99$)---but SimNPO buys that forget quality with an MU of $0.58$, whereas Nullify gives up nothing. On the harder \texttt{forget10} split Nullify dominates every baseline: it raises FQ from SimNPO's $0.45$ to $0.58$ and holds MU at $0.62$---the highest among unlearning methods, tied with SimNPO and equal to Retrain. The per-set numbers (\texttt{forget10}) confirm the absence of collateral damage: Nullify's ROUGE-L on World Facts ($0.90$) and the Retain set ($0.95$) tracks the Original model ($0.91$/$0.98$) closely.


\subsubsection{Results on MUSE-News}

On MUSE-News (Table~\ref{tab:muse}), the privacy metric PrivLeak is the most discriminative: it is near $0$ only when a method removes the target knowledge to the same degree as Retrain, and diverges under both under- and over-unlearning.

The baselines fail in opposite directions. The preference- and gradient-based methods leak heavily---NPO ($108.91$), GradDiff ($108.12$), and SimNPO ($72.93$) sit far above $0$---while Task Vector swings to the opposite extreme ($-100.00$), over-correcting past the retrained reference. The representation baseline LUNAR barely moves the needle: its PrivLeak of $-99.80$ is essentially the un-unlearned level of the Original model ($-98.71$). Gradient Ascent collapses outright, with KM-D$_\text{r}$ at $0.00$.

Nullify is the only method that lands near the ideal: a PrivLeak of $-5.18$, an order of magnitude closer to $0$ than any baseline. Its forget-set memorization KM-D$_\text{f}$ drops from the Original's $0.63$ to $0.27$, close to the Retrain reference ($0.33$), while retain-set knowledge KM-D$_\text{r}$ stays at $0.51$, near Retrain ($0.54$) and the Original ($0.54$). Nullify thus removes the target knowledge to a Retrain-like degree without the privacy leakage or the utility collapse of the baselines.

\begin{table}[!t]
  \centering\small
  \setlength{\tabcolsep}{5pt}
  \begin{tabular}{l | c c c}
    \toprule
    \multirow{2}{*}{\textbf{Method}}
      & \textbf{Unlearning} & \textbf{Privacy} & \textbf{Utility} \\
      & KM-D$_\text{f}$ & PrivLeak & KM-D$_\text{r}$ \\
    \midrule
    Original$^\ddagger$        & 0.63 & $-$98.71 & 0.54 \\
    Retrain$^{\dagger,\ddagger}$ & 0.33 & \phantom{$-$}0.00 & 0.54 \\
    \midrule
    GA$^\ddagger$              & 0.00 & \phantom{$-$}20.14  & 0.00 \\
    GradDiff$^\ddagger$        & 0.31 & 108.12              & 0.28 \\
    Task Vector$^\ddagger$     & 0.59 & $-$100.00           & 0.48 \\
    NPO$^\ddagger$             & 0.57 & 108.91              & 0.38 \\
    SimNPO$^\ddagger$          & 0.45 & \phantom{$-$}72.93  & 0.40 \\
    LUNAR           & 0.55 & $-$99.80            & 0.52 \\
    \midrule
    \textbf{Nullify} & 0.27 & \textbf{$-$5.18} & 0.51 \\
    \bottomrule
  \end{tabular}
  \caption{Results on MUSE-News~\cite{shi2024muse}. KM-D$_\text{f}$/KM-D$_\text{r}$: forget-/retain-set knowledge memorization (best when closest to Retrain$^\dagger$, $0.33$/$0.54$); PrivLeak: privacy leakage (ideal $=0$, \textbf{bold}). Rows marked with $^\ddagger$ are taken from~\citet{fan2025simplicity}.}
  \label{tab:muse}
\end{table}

\subsection{Robustness Experiments}
\label{sec:robustness}

We evaluate extraction robustness with approximately 11{,}200 prompts spanning six attack families and 14 variants: prefix injection, roleplay, few-shot elicitation, partial-answer elicitation, multi-turn extraction, and cloze/multiple-choice prompting. A separate test covers question paraphrasing and cloze-style reformatting. Tables~\ref{tab:robust-attacks} and~\ref{tab:robust-formats} report exact-span leakage---the fraction of forget-example outputs containing a gold answer span---where lower is better. The six-family score is an unweighted macro-average; multi-turn leakage is measured on the final extraction turn.

\begin{table}[!t]
  \centering
  \small
  \setlength{\tabcolsep}{2pt}
  \renewcommand{\arraystretch}{0.95}
  \resizebox{\columnwidth}{!}{%
  \begin{tabular}{@{}lrrr@{}}
    \toprule
    \textbf{Attack family} & \textbf{Original} & \textbf{Retrain} & \textbf{Nullify} \\
    \midrule
    Prefix injection           & 0.633 & 0.087 & 0.222 \\
    Roleplay                   & 0.434 & 0.071 & 0.111 \\
    Few-shot elicitation       & 0.388 & 0.074 & \textbf{0.014} \\
    Partial-answer elicitation & 0.202 & 0.042 & 0.065 \\
    Multi-turn extraction      & 0.405 & 0.085 & 0.128 \\
    Cloze / multiple choice    & 0.285 & 0.210 & \textbf{0.030} \\
    \midrule
    \textbf{Six-family mean}   & \textbf{0.391} & \textbf{0.095} & \textbf{0.095} \\
    \bottomrule
  \end{tabular}%
  }
  \caption{Exact-span leakage under adversarial extraction (lower is better). The mean is unweighted across families; multi-turn extraction is scored on the final extraction turn.}
  \label{tab:robust-attacks}
\end{table}

\begin{table}[!t]
  \centering
  \small
  \setlength{\tabcolsep}{3pt}
  \begin{tabular}{@{}lrrr@{}}
    \toprule
    \textbf{Prompt type} & \textbf{Original} & \textbf{Retrain} & \textbf{Nullify} \\
    \midrule
    Plain QA             & 0.975 & 0.085 & 0.253 \\
    Paraphrased question & 0.478 & 0.128 & 0.218 \\
    Cloze completion     & 0.358 & 0.153 & \textbf{0.060} \\
    \bottomrule
  \end{tabular}
  \caption{Exact-span leakage under paraphrasing and reformatting (lower is better).}
  \label{tab:robust-formats}
\end{table}

Nullify lowers leakage relative to Original in every family, reducing the macro-average from $0.391$ to $0.095$, which matches Retrain at the reported three-decimal precision. On this metric, Nullify yields lower leakage than Retrain for few-shot elicitation ($0.014$ vs.\ $0.074$), cloze/multiple choice ($0.030$ vs.\ $0.210$), and cloze completion ($0.060$ vs.\ $0.153$). Under paraphrasing, leakage remains below Original ($0.218$ vs.\ $0.478$), indicating partial transfer to nearby free-form query formats.


\subsection{Compute and Memory Efficiency}
\label{sec:cost}

Beyond unlearning quality, deployment cost determines whether a method can serve real-world right-to-be-forgotten requests. We benchmark every method end-to-end on identical hardware (a single NVIDIA H800 80\,GB) under the TOFU \texttt{forget05} / Llama-2-7B-chat-hf setting, reporting peak GPU memory and the wall-clock paid per single \{hyperparameter, seed\} configuration. Wall-clock follows each method's paper recipe: 5 epochs for GA, GradDiff and IDK~\cite{maini2024tofu}; 20 epochs for NPO and 10 for SimNPO~\cite{fan2025simplicity}; 20 epochs of single-layer LoRA training for LUNAR~\cite{shen2025lunar}; and a single closed-form preparation plus a per-$\alpha$ inference pass for Nullify. Reported times include training and the default post-hoc TOFU evaluation, so the comparison is end-to-end.

\begin{table}[!t]
  \centering\small
  \setlength{\tabcolsep}{4pt}
  \begin{tabular}{l | c c c r}
    \toprule
    \textbf{Method} & \textbf{\#Train} & \textbf{Mem} & \textbf{Wall-clock} & \textbf{$\times$ vs} \\
                    & \textbf{params}  & \textbf{(GB)} & \textbf{(min)}      & \textbf{Nullify} \\
    \midrule
    Retrain$^\dagger$ & 6.6\,B & $\sim$79 & $\sim$300 & 57.7$\times$ \\
    \midrule
    GA           & 6.6\,B & 79 & 14.5  &  2.8$\times$ \\
    GradDiff     & 6.6\,B & 79 & 13.5  &  2.6$\times$ \\
    IDK          & 6.6\,B & 79 & $\sim$13 & 2.5$\times$ \\
    NPO          & 6.6\,B & 79 & 58.2  & 11.2$\times$ \\
    SimNPO       & 6.6\,B & 79 & 27.4  &  5.3$\times$ \\
    LUNAR        & 66\,K  & 16 &  6.3  &  1.2$\times$ \\
    \midrule
    \textbf{Nullify} & \textbf{0} & \textbf{14} & \textbf{5.2} & \textbf{1.0$\times$} \\
    \bottomrule
  \end{tabular}
  \caption{End-to-end cost on TOFU \texttt{forget05} / Llama-2-7B-chat-hf, single NVIDIA H800. \#Train params: parameters updated by gradient descent. Wall-clock: per \{hyperparameter, seed\} configuration at each method's paper-recipe epoch budget.}
  \label{tab:cost}
\end{table}

Table~\ref{tab:cost} summarizes the result. Every full-finetune baseline---GA, GradDiff, NPO, SimNPO---carries a 6.6\,B-parameter Adam state, saturates GPU memory at ${\sim}79$\,GB, and takes 13--58 minutes per configuration. Nullify instead runs in pure inference memory (14\,GB) and finishes one configuration in 5.2 minutes: 3.2 minutes to extract activations and solve the closed-form ridge once, plus 1.95 minutes of inference-time evaluation at one $\alpha$. Nullify dominates five of the six baselines outright---GA, GradDiff, IDK, NPO and LUNAR are each both slower and weaker (Nullify, e.g., surpasses LUNAR in forget quality, $0.96$ vs.\ $0.87$, at lower cost, $5.2$ vs.\ $6.3$ min). Only SimNPO and the Retrain oracle reach higher forget quality, and they pay $5.3\times$ and ${\sim}58\times$ Nullify's wall-clock to do so. Nullify is therefore Pareto-optimal in the forget-quality / wall-clock plane: no baseline is simultaneously faster and stronger. Moreover, since $\alpha$ is a runtime scalar that does not enter the closed-form preparation, adapting Nullify to a new forget request at deployment requires only a single inference-time evaluation per $\alpha$ value, with no retraining.


\subsection{Analysis}


\subsubsection{Ablation Study}

Nullify combines two components: the null-space projection that confines the steering operator, and the IDK-contrastive forget direction. Table~\ref{tab:ablation} removes each in turn on \texttt{forget05}. Replacing the null-space projector with the identity ($\hat{\mathbf{P}}{=}\mathbf{I}$) removes the utility guarantee: even at its most conservative strength ($\alpha{=}-0.3$) MU is only $0.36$, and by $\alpha{=}-1.0$ it has collapsed to $0.01$---there is no steering strength at which an unconstrained operator both forgets and preserves utility. Replacing the IDK-contrastive direction with a random vector of matched norm halves forget quality (FQ $0.96\!\to\!0.466$) while leaving MU near $0.60$. The two ablations isolate the two roles cleanly: the null-space projection is what preserves utility, and the IDK-contrastive direction is what carries the forget signal.

\subsubsection{Hyperparameter Sensitivity}
\label{sec:sensitivity}

Tables~\ref{tab:sweep-lambda}--\ref{tab:sweep-p} in Appendix~\ref{app:sensitivity} sweep the three scalars of Nullify on \texttt{forget05}. Forget quality peaks at the operating point $\lambda{=}10$, $\alpha{=}-2.0$, with $\alpha$ sitting inside a wide robust plateau---FQ exceeds $0.92$ throughout $\alpha\in[-2.3,-1.9]$ rather than on a knife-edge. Critically, MU stays within $0.61$--$0.63$ across \emph{every} value of $p$, direct evidence that utility preservation is delivered by the null-space construction, not by hyperparameter tuning.

\subsubsection{Forget-Quality Visualization}
\label{sec:tr_viz}

To make the statistical content of FQ concrete, Figure~\ref{fig:tr_dist} overlays the per-example truth-ratio densities of the Original model, the Nullify-steered model, and the Retrain oracle on the \texttt{forget05} and \texttt{forget10} forget sets. FQ is by definition the Kolmogorov--Smirnov $p$-value between the steered and Retrain distributions, so a perfect match gives $p\!\to\!1$ and a memorized forget set shows up as a narrow low-truth-ratio mode.

The Original model exhibits exactly this memorization signature---a sharp peak near $0.56$--$0.58$ that places most forget questions far outside the Retrain envelope. Nullify lifts this mass back into a broad distribution centered at $0.86$ (\texttt{forget05}) and $0.87$ (\texttt{forget10}), nearly overlapping the Retrain reference at $0.90$ and $0.91$. The resulting KS $p$-values, $0.96$ and $0.58$, are the FQ scores of Table~\ref{tab:tofu}. The visualization confirms that Nullify does not merely lower a verbatim-recall score; it restores the \emph{distributional} indistinguishability from Retrain that FQ is designed to measure. Representative qualitative \texttt{forget05} outputs, contrasting Nullify's forget-set abstention with its unchanged retain-set generations, are provided in Appendix~\ref{app:case} (Table~\ref{tab:case}).

\begin{table}[!t]
  \centering\small
  \setlength{\tabcolsep}{5pt}
  \resizebox{\columnwidth}{!}{%
  \begin{tabular}{l | c c c}
    \toprule
    \textbf{Variant} & \textbf{best $\alpha$} & \textbf{FQ} & \textbf{MU} \\
    \midrule
    \textbf{Nullify (full)} & $-2.0$ & \textbf{0.965} & \textbf{0.622} \\
    \;\; w/o null-space proj.\ ($\hat{\mathbf{P}}{=}\mathbf{I}$) & $-0.3$ & 0.270 & 0.360 \\
    \;\; w/o IDK-contrastive dir.\ (random $\mathbf{r}$) & $-2.0$ & 0.466 & 0.600 \\
    \bottomrule
  \end{tabular}%
  }
  \caption{Ablation on TOFU \texttt{forget05}. The $\hat{\mathbf{P}}{=}\mathbf{I}$ row is shown at its own best $\alpha$ ($-0.3$); by $\alpha{=}-1.0$ it has already collapsed to FQ $0.00$ / MU $0.01$.}
  \label{tab:ablation}
\end{table}

\begin{figure}[t]
\centering
\includegraphics[width=\linewidth]{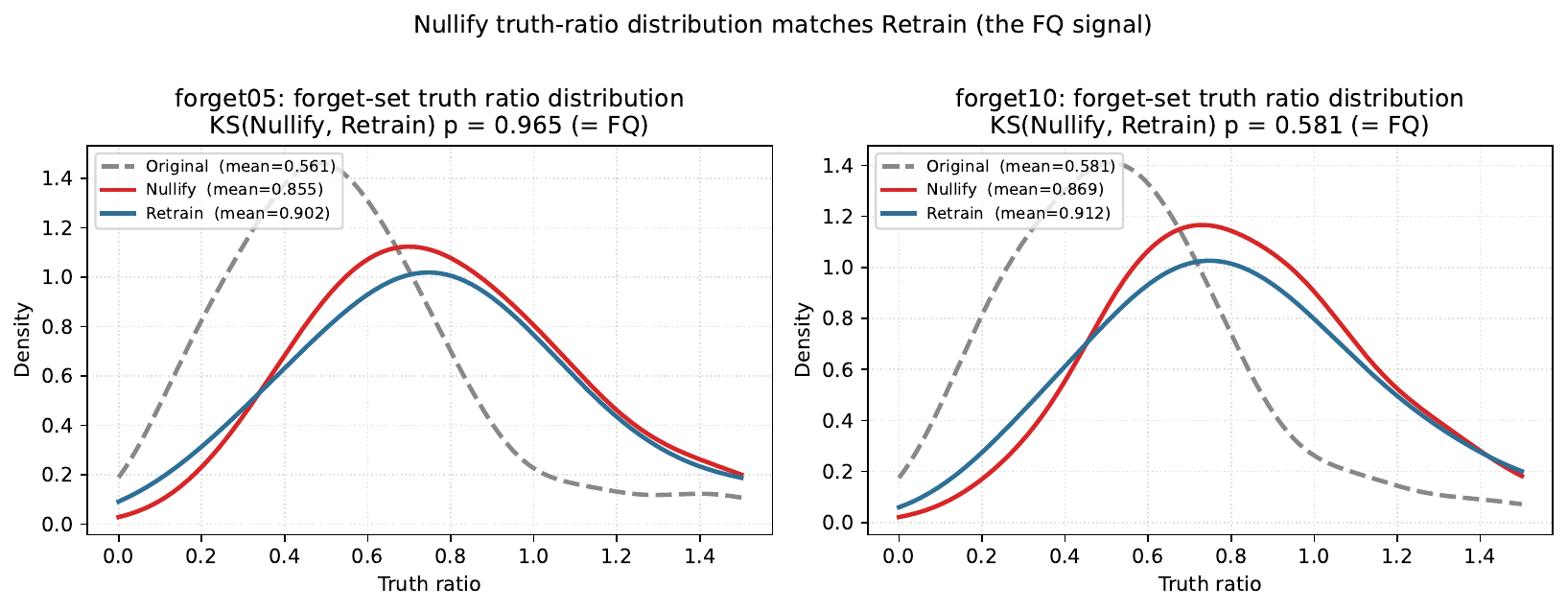}
\caption{Forget-set truth-ratio density under three models on TOFU. Dashed gray: Original (memorized); red: Nullify; blue: Retrain (oracle). Nullify collapses the memorization peak back onto the Retrain envelope; the Kolmogorov--Smirnov $p$-value in each title is the Forget Quality used in Table~\ref{tab:tofu}.}
\label{fig:tr_dist}
\end{figure}

\section{Conclusion}

We present Nullify, a training-free, null-space-constrained activation-steering framework for LLM unlearning. By steering forget queries off their memorized-answer trajectory while confining every correction to the null space of retained-knowledge activations, Nullify decouples forgetting from utility preservation rather than trading one against the other---without modifying any model parameter. On TOFU and MUSE, it matches or surpasses the strongest gradient-based, preference-optimization, and representation-intervention baselines in forget quality while keeping utility near-lossless, at a small fraction of their compute and memory. More broadly, Nullify shifts LLM unlearning from destructive parameter reshaping to lightweight, geometrically constrained activation steering---a plug-and-play path to the "right to be forgotten." Future work will extend it to multimodal architectures and long-context unlearning.

\section*{Limitations}

As an inference-time intervention, Nullify suppresses access to memorized answers while its operator is active rather than permanently altering the underlying weights. Our robustness evaluation covers fixed-template black-box extraction and nearby free-form reformulations in the intended serving-time setting. Adaptive attacks and white-box settings in which the operator can be inspected, modified, or disabled remain outside the current scope. The intervention layers and steering strength are currently selected per benchmark; automatic configuration across architectures and deployment settings remains an important direction for future work.

\section*{Ethics Considerations}

Nullify is designed to support privacy-, copyright-, and safety-driven removal of memorized content from deployed LLMs, in line with the right to erasure under the GDPR and the data-governance requirements of the EU AI Act. Because Nullify avoids weight updates and gradient computation, it reduces the energy and compute cost of fulfilling a deletion request by roughly an order of magnitude relative to gradient-based unlearning, which means providers serving large user bases can afford to honor every deletion request. We caution, however, that the same mechanism that suppresses sensitive memorized content can in turn be configured to hide outputs that should remain public---for example, silently filtering information about public events, public figures, or known model behaviors. To guard against such misuse, we restrict our evaluation to published, peer-reviewed unlearning benchmarks (TOFU and MUSE), whose forget sets target privacy- and copyright-motivated removal rather than viewpoint suppression; we further recommend that any party deploying this method publicly disclose the categories of content their steering operators are configured to remove, so that the suppression remains subject to external oversight.


\bibliography{custom}

@inproceedings{carlini2021extracting,
  title     = {Extracting Training Data from Large Language Models},
  author    = {Nicholas Carlini and Florian Tramer and Eric Wallace and Matthew Jagielski
               and Ariel Herbert-Voss and Katherine Lee and Adam Roberts and Tom Brown
               and Dawn Song and Ulfar Erlingsson and Alina Oprea and Colin Raffel},
  booktitle = {30th USENIX Security Symposium},
  pages     = {2633--2650},
  year      = {2021}
}

@inproceedings{bourtoule2021machine,
  title={Machine Unlearning},
  author={Bourtoule, Lucas and Chandrasekaran, Varun and Choquette-Choo, Christopher A
          and Jia, Hengrui and Travers, Adelin and Zhang, Baiwu and Lie, David and Papernot, Nicolas},
  booktitle={2021 IEEE Symposium on Security and Privacy (SP)},
  pages={141--159},
  year={2021},
  organization={IEEE}
}

@inproceedings{thudi2022unrolling,
  title={Unrolling {SGD}: Understanding Factors Influencing Machine Unlearning},
  author={Thudi, Anvith and Deza, Gabriel and Chandrasekaran, Varun and Papernot, Nicolas},
  booktitle={2022 IEEE 7th European Symposium on Security and Privacy (EuroS\&P)},
  pages={303--319},
  year={2022},
  organization={IEEE}
}

@inproceedings{kurmanji2024towards,
 author = {Kurmanji, Meghdad and Triantafillou, Peter and Hayes, Jamie and Triantafillou, Eleni},
 booktitle = {Advances in Neural Information Processing Systems},
 pages = {1957--1987},
 publisher = {Curran Associates, Inc.},
 title = {Towards Unbounded Machine Unlearning},
 url = {https://proceedings.neurips.cc/paper_files/paper/2023/file/062d711fb777322e2152435459e6e9d9-Paper-Conference.pdf},
 volume = {36},
 year = {2023}
}

@inproceedings{jang2022knowledge,
  title={Knowledge Unlearning for Mitigating Privacy Risks in Language Models},
  author={Jang, Joel and Yoon, Dongkeun and Yang, Sohee and Cha, Sungmin and Lee, Moontae and Logeswaran, Lajanugen and Seo, Minjoon},
  booktitle={Proceedings of the 61st Annual Meeting of the Association for Computational Linguistics (Volume 1: Long Papers)},
  pages={14389--14408},
  year={2023}
}

@inproceedings{yao2023large,
  title={Large Language Model Unlearning},
  author={Yao, Yuanshun and Xu, Xiaojun and Liu, Yang},
  booktitle={Advances in Neural Information Processing Systems (NeurIPS)},
  year={2024}
}

@inproceedings{
zhang2024negative,
title={Negative Preference Optimization: From Catastrophic Collapse to Effective Unlearning},
author={Ruiqi Zhang and Licong Lin and Yu Bai and Song Mei},
booktitle={First Conference on Language Modeling},
year={2024},
url={https://openreview.net/forum?id=MXLBXjQkmb}
}

@inproceedings{fan2025simplicity,
  title={Simplicity Prevails: Rethinking Negative Preference Optimization for {LLM} Unlearning},
  author={Fan, Chongyu and Liu, Jiancheng and Lin, Licong and Jia, Jinghan
          and Zhang, Ruiqi and Mei, Song and Liu, Sijia},
  booktitle={Advances in Neural Information Processing Systems},
  year={2025}
}

@inproceedings{li2024wmdp,
  title = {The WMDP Benchmark: Measuring and Reducing Malicious Use With Unlearning},
  author = {Nathaniel Li and Alexander Pan and Anjali Gopal and Summer Yue and Daniel Berrios and Alice Gatti and Justin D. Li and Ann-Kathrin Dombrowski and Shashwat Goel and Gabriel Mukobi and Nathan Helm-Burger and Rassin R. Lababidi and Lennart Justen and Andrew B. Liu and Michael Chen and Isabelle Barrass and Oliver Zhang and Xiaoyuan Zhu and Rishub Tamirisa and Bhrugu Bharathi and Ariel Herbert-Voss and Cort B. Breuer and Andy Zou and Mantas Mazeika and Zifan Wang and Palash Oswal and Weiran Lin and Adam A. Hunt and Justin Tienken-Harder and Kevin Y. Shih and Kemper Talley and John Guan and Ian Steneker and David Campbell and Brad Jokubaitis and Steven Basart and Stephen Fitz and P. Kumaraguru and K. Karmakar and U. Tupakula and Vijay Varadharajan and Yan Shoshitaishvili and Jimmy Ba and K. Esvelt and Alexandr Wang and Dan Hendrycks},
  booktitle = {Proceedings of the 41st International Conference on Machine Learning},
  year = {2024},
  journal = {ICML}
}

@inproceedings{dang2025effects,
  title={On Effects of Steering Latent Representation for Large Language Model Unlearning},
  author={Dang, Huu-Tien and Pham, Tin and Thanh-Tung, Hoang and Inoue, Naoya},
  booktitle={Proceedings of the AAAI Conference on Artificial Intelligence},
  volume={39},
  pages={23733--23742},
  year={2025}
}

@inproceedings{meng2022locating,
  title={Locating and Editing Factual Associations in {GPT}},
  author={Meng, Kevin and Bau, David and Andonian, Alex and Belinkov, Yonatan},
  booktitle={Advances in Neural Information Processing Systems},
  volume={35},
  pages={17359--17372},
  year={2022}
}

@inproceedings{geva2021transformer,
  title     = {Transformer Feed-Forward Layers Are Key-Value Memories},
  author    = {Geva, Mor and Schuster, Roei and Berant, Jonathan and Levy, Omer},
  booktitle = {Proceedings of the 2021 Conference on Empirical Methods in Natural Language Processing},
  pages     = {5484--5495},
  year      = {2021},
  publisher = {Association for Computational Linguistics},
  doi       = {10.18653/v1/2021.emnlp-main.446},
  url       = {https://aclanthology.org/2021.emnlp-main.446}
}

@inproceedings{fang2025alphaedit,
  title={Alphaedit: Null-space constrained knowledge editing for language models},
  author={Fang, Junfeng and Jiang, Houcheng and Wang, Kun and Ma, Yunshan and Shi, Jie and Wang, Xiang and He, Xiangnan and Chua, Tat-Seng},
  booktitle={International Conference on Learning Representations},
  volume={2025},
  pages={16366--16396},
  year={2025}
}

@inproceedings{maini2024tofu,
  title={{TOFU}: A Task of Fictitious Unlearning for {LLMs}},
  author={Maini, Pratyush and Feng, Zhili and Schwarzschild, Avi
          and Lipton, Zachary C and Kolter, J Zico},
  booktitle={Proceedings of the First Conference on Language Modeling (COLM)},
  year={2024}
}

@inproceedings{
shi2024muse,
title={{MUSE}: Machine Unlearning Six-Way Evaluation for Language Models},
author={Weijia Shi and Jaechan Lee and Yangsibo Huang and Sadhika Malladi and Jieyu Zhao and Ari Holtzman and Daogao Liu and Luke Zettlemoyer and Noah A. Smith and Chiyuan Zhang},
booktitle={The Thirteenth International Conference on Learning Representations},
year={2025},
url={https://openreview.net/forum?id=TArmA033BU}
}

@inproceedings{rimsky2024steering,
  title={Steering {Llama 2} via Contrastive Activation Addition},
  author={Rimsky, Nina and Gabrieli, Nick and Schulz, Julian and Tong, Meg
          and Hubinger, Evan and Turner, Alexander Matt},
  booktitle={Proceedings of the 62nd Annual Meeting of the Association for Computational Linguistics},
  pages={15504--15522},
  year={2024}
}

@article{zou2023representation,
  title={Representation Engineering: A Top-Down Approach to {AI} Transparency},
  author={Zou, Andy and Phan, Long and Chen, Sarah and Campbell, James and Guo, Phillip and Ren, Richard and Pan, Alexander and Yin, Xuwang and Mazeika, Mantas and Dombrowski, Ann-Kathrin and Goel, Shashwat and Li, Nathaniel and Byun, Michael J. and Wang, Zifan and Mallen, Alex and Basart, Steven and Koyejo, Sanmi and Song, Dawn and Fredrikson, Matt and Kolter, J. Zico and Hendrycks, Dan},
  journal={arXiv preprint arXiv:2310.01405},
  year={2023}
}

@inproceedings{arditi2024refusal,
  title={Refusal in Language Models Is Mediated by a Single Direction},
  author={Arditi, Andy and Obeso, Oscar and Syed, Aaquib and Paleka, Daniel
          and Panickssery, Nina and Gurnee, Wes and Nanda, Neel},
  booktitle={Advances in Neural Information Processing Systems},
  year={2024}
}

@inproceedings{sheng2025alphasteer,
  title={{AlphaSteer}: Learning Refusal Steering with Principled Null-Space Constraint},
  author={Sheng, Leheng and Shen, Changshuo and Zhao, Weixiang and Fang, Junfeng
          and Liu, Xiaohao and Liang, Zhenkai and Wang, Xiang and Zhang, An and Chua, Tat-Seng},
  booktitle={The Fourteenth International Conference on Learning Representations},
  year={2026}
}

@inproceedings{wang2021training,
  title={Training Networks in Null Space of Feature Covariance for Continual Learning},
  author={Wang, Shipeng and Li, Xiaorong and Sun, Jian and Xu, Zongben},
  booktitle={Proceedings of the IEEE/CVF Conference on Computer Vision and Pattern Recognition},
  pages={184--193},
  year={2021}
}

@inproceedings{shen2025lunar,
  title={{LUNAR}: {LLM} Unlearning via Neural Activation Redirection},
  author={Shen, William F and Qiu, Xinchi and Kurmanji, Meghdad and Iacob, Alex
          and Sani, Lorenzo and Chen, Yihong and Cancedda, Nicola and Lane, Nicholas D},
  booktitle={Advances in Neural Information Processing Systems},
  year={2025}
}

@article{rafailov2023direct,
  title={Direct preference optimization: Your language model is secretly a reward model},
  author={Rafailov, Rafael and Sharma, Archit and Mitchell, Eric and Manning, Christopher D and Ermon, Stefano and Finn, Chelsea},
  journal={Advances in neural information processing systems},
  volume={36},
  pages={53728--53741},
  year={2023}
}

@inproceedings{park2023linear,
    author = {Park, Kiho and Choe, Yo Joong and Veitch, Victor},
    title = {The linear representation hypothesis and the geometry of large language models},
    booktitle = {Proceedings of the 41st International Conference on Machine Learning},
    series = {ICML'24},
    location = {Vienna, Austria},
    year = {2024},
    publisher = {JMLR.org}
}

@article{wang2024unlearning,
  title={Unlearning with Control: Assessing Real-World Utility for Large Language Model Unlearning},
  author={Wang, Qizhou and Han, Bo and Yang, Puning and Zhu, Jianing and Liu, Tongliang and Sugiyama, Masashi},
  journal={arXiv preprint arXiv:2406.09179},
  year={2024}
}

@misc{eu_ai_act_2024,
  title={Regulation (EU) 2024/1689 of the European Parliament and of the Council of 13 June 2024 laying down harmonised rules on artificial intelligence (Artificial Intelligence Act)},
  author={{European Parliament and Council of the European Union}},
  year={2024},
  howpublished={Official Journal of the European Union, OJ L, 2024/1689},
  url={https://eur-lex.europa.eu/legal-content/EN/TXT/?uri=CELEX:32024R1689}
}

@inproceedings{liu2022few,
  title     = {Continual Learning and Private Unlearning},
  author    = {Bo Liu and Qiang Liu and Peter Stone},
  booktitle = {Proceedings of The 1st Conference on Lifelong Learning Agents},
  series    = {Proceedings of Machine Learning Research},
  volume    = {199},
  pages     = {243--254},
  year      = {2022},
  publisher = {PMLR},
  url       = {https://proceedings.mlr.press/v199/liu22a.html}
}

@inproceedings{ilharco2022editing,
  title={Editing Models with Task Arithmetic},
  author={Ilharco, Gabriel and Ribeiro, Marco Tulio and Wortsman, Mitchell and Gururangan, Suchin and Schmidt, Ludwig and Hajishirzi, Hannaneh and Farhadi, Ali},
  booktitle={The Eleventh International Conference on Learning Representations (ICLR)},
  year={2023}
}

@article{eldan2023s,
  title={Who's harry potter? approximate unlearning in llms},
  author={Eldan, Ronen and Russinovich, Mark},
  journal={arXiv preprint arXiv:2310.02238},
  year={2023}
}

@article{deng2025guard,
  title={Guard: Generation-time llm unlearning via adaptive restriction and detection},
  author={Deng, Zhijie and Liu, Chris Yuhao and Pang, Zirui and He, Xinlei and Feng, Lei and Xuan, Qi and Zhu, Zhaowei and Wei, Jiaheng},
  journal={arXiv preprint arXiv:2505.13312},
  year={2025}
}

@article{yang2025qwen3,
  title={Qwen3 Technical Report},
  author={An Yang and Anfeng Li and Baosong Yang and Beichen Zhang and
          Binyuan Hui and Bo Zheng and Bowen Yu and Chang Gao and
          Chengen Huang and Chenxu Lv and Chujie Zheng and Dayiheng Liu and
          Fan Zhou and Fei Huang and Feng Hu and Hao Ge and Haoran Wei and
          Huan Lin and Jialong Tang and Jian Yang and Jianhong Tu and
          Jianwei Zhang and Jianxin Yang and Jiaxi Yang and Jing Zhou and
          Jingren Zhou and Junyang Lin and Kai Dang and Keqin Bao and
          Kexin Yang and Le Yu and Lianghao Deng and Mei Li and Mingfeng Xue and
          Mingze Li and Pei Zhang and Peng Wang and Qin Zhu and Rui Men and
          Ruize Gao and Shixuan Liu and Shuang Luo and Tianhao Li and
          Tianyi Tang and Wenbiao Yin and Xingzhang Ren and Xinyu Wang and
          Xinyu Zhang and Xuancheng Ren and Yang Fan and Yang Su and
          Yichang Zhang and Yinger Zhang and Yu Wan and Yuqiong Liu and
          Zekun Wang and Zeyu Cui and Zhenru Zhang and Zhipeng Zhou and
          Zihan Qiu},
  journal={arXiv preprint arXiv:2505.09388},
  year={2025},
  doi={10.48550/arXiv.2505.09388},
  url={https://arxiv.org/abs/2505.09388}
}

\newpage
\appendix

\section{Theoretical Details}
\label{app:theory}
 
This appendix supplies (i) the proof of
Lemma~\ref{lem:null_equiv} and (ii) the full derivation of
Eq.~\eqref{eq:nullify_ridge} with a careful treatment of the
rank deficiency of $\hat{\mathbf{P}}$.
Throughout we suppress the layer index $(\ell)$.
 
\subsection{Proof of Lemma~\ref{lem:null_equiv}}
\label{app:lemma_proof}
 
\begin{proof}
$(\subseteq)$ If $\mathbf{x}^{\!\top}\mathbf{H}=\mathbf{0}$, then
$\mathbf{x}^{\!\top}\mathbf{H}\mathbf{H}^{\!\top}=\mathbf{0}$.
 
$(\supseteq)$ If $\mathbf{x}^{\!\top}\mathbf{H}\mathbf{H}^{\!\top}=\mathbf{0}$, then
\begin{equation*}
\mathbf{x}^{\!\top}\mathbf{H}\mathbf{H}^{\!\top}\mathbf{x}
=\bigl\lVert\mathbf{H}^{\!\top}\mathbf{x}\bigr\rVert_{2}^{2}=0,
\end{equation*}
so $\mathbf{H}^{\!\top}\mathbf{x}=\mathbf{0}$, i.e.,
$\mathbf{x}^{\!\top}\mathbf{H}=\mathbf{0}$.
\end{proof}
 
The equivalence lets Nullify replace an SVD of the
$D{\times}N_k$ matrix $\mathbf{H}_k$ by an eigendecomposition of
the $D{\times}D$ Gram $\mathbf{H}_k\mathbf{H}_k^{\!\top}$, which is the
key to the $O(N_k D^{2}{+}D^{3})$ offline cost in the regime
$N_k\gg D$ relevant here.
 
\subsection{Derivation of the Closed-Form Solution}
\label{app:closed_form}
 
The objective in Eq.~\eqref{eq:nullify_objective} is
\begin{equation}
\label{eq:appA_J}
J(\tilde{\boldsymbol{\Delta}})
=\bigl\lVert\tilde{\boldsymbol{\Delta}}\hat{\mathbf{P}}\mathbf{H}_{f}-\mathbf{R}\bigr\rVert_{F}^{2}
+\lambda\bigl\lVert\tilde{\boldsymbol{\Delta}}\hat{\mathbf{P}}\bigr\rVert_{F}^{2}.
\end{equation}
Using $\hat{\mathbf{P}}^{\!\top}=\hat{\mathbf{P}}=\hat{\mathbf{P}}^{2}$,
setting $\nabla J=\mathbf{0}$ gives the normal equation
\begin{equation}
\label{eq:appA_normal}
\tilde{\boldsymbol{\Delta}}\,\mathbf{A}
=\mathbf{R}\mathbf{H}_{f}^{\!\top}\hat{\mathbf{P}},
\end{equation}
with
$\mathbf{A}:=\hat{\mathbf{P}}\mathbf{H}_{f}\mathbf{H}_{f}^{\!\top}\hat{\mathbf{P}}+\lambda\hat{\mathbf{P}}$.
 
\paragraph{Rank deficiency.}
$\mathbf{A}$ has rank $k=\mathrm{rank}(\hat{\mathbf{P}})<D$.
$\mathbf{A}\mathbf{x}=\mathbf{0}$ implies
$\mathbf{x}^{\!\top}\mathbf{A}\mathbf{x}
=\lVert\mathbf{H}_{f}^{\!\top}\hat{\mathbf{P}}\mathbf{x}\rVert^{2}
+\lambda\lVert\hat{\mathbf{P}}\mathbf{x}\rVert^{2}=0$,
which forces $\hat{\mathbf{P}}\mathbf{x}=\mathbf{0}$ since $\lambda>0$.
Together with the obvious converse,
\begin{equation}
\label{eq:appA_nullrange}
\mathrm{Null}(\mathbf{A})=\mathrm{Null}(\hat{\mathbf{P}}),
\quad
\mathrm{range}(\mathbf{A})=\mathrm{range}(\hat{\mathbf{P}}).
\end{equation}
The right-hand side of \eqref{eq:appA_normal} has rows in
$\mathrm{range}(\hat{\mathbf{P}})$ (since $\mathbf{b}^{\!\top}\hat{\mathbf{P}}=\mathbf{b}^{\!\top}$
for any such row $\mathbf{b}^{\!\top}$), hence in
$\mathrm{range}(\mathbf{A})$: the system is consistent, and the
minimum-Frobenius-norm solution is
\begin{equation}
\label{eq:appA_pinv}
\tilde{\boldsymbol{\Delta}}^{\star}
=\mathbf{R}\mathbf{H}_{f}^{\!\top}\hat{\mathbf{P}}\,\mathbf{A}^{+},
\end{equation}
which is Eq.~\eqref{eq:nullify_ridge}.
 
\paragraph{Verification via a reduced full-rank form.}
Write $\hat{\mathbf{P}}=\hat{\mathbf{V}}\hat{\mathbf{V}}^{\!\top}$ with
$\hat{\mathbf{V}}\in\mathbb{R}^{D\times k}$,
$\hat{\mathbf{V}}^{\!\top}\hat{\mathbf{V}}=\mathbf{I}_{k}$. Set
$\mathbf{Y}=\tilde{\boldsymbol{\Delta}}\hat{\mathbf{V}}$,
$\tilde{\mathbf{H}}_{f}=\hat{\mathbf{V}}^{\!\top}\mathbf{H}_{f}$. Then
$\tilde{\boldsymbol{\Delta}}\hat{\mathbf{P}}=\mathbf{Y}\hat{\mathbf{V}}^{\!\top}$,
$\lVert\mathbf{Y}\hat{\mathbf{V}}^{\!\top}\rVert_{F}^{2}=\lVert\mathbf{Y}\rVert_{F}^{2}$,
and \eqref{eq:appA_J} reduces to
\begin{equation}
\label{eq:appA_reduced}
J=\bigl\lVert\mathbf{Y}\tilde{\mathbf{H}}_{f}-\mathbf{R}\bigr\rVert_{F}^{2}
+\lambda\lVert\mathbf{Y}\rVert_{F}^{2},
\end{equation}
a strictly convex ridge problem with unique minimizer
\begin{equation}
\label{eq:appA_Ystar}
\mathbf{Y}^{\star}
=\mathbf{R}\tilde{\mathbf{H}}_{f}^{\!\top}
\bigl(\tilde{\mathbf{H}}_{f}\tilde{\mathbf{H}}_{f}^{\!\top}+\lambda\mathbf{I}_{k}\bigr)^{-1}.
\end{equation}
The inverse is a genuine $k{\times}k$ inverse since
$\lambda>0$. Substituting $\hat{\mathbf{P}}=\hat{\mathbf{V}}\hat{\mathbf{V}}^{\!\top}$
into $\mathbf{A}$ yields
$\mathbf{A}=\hat{\mathbf{V}}\mathbf{S}\hat{\mathbf{V}}^{\!\top}$
with
$\mathbf{S}=\tilde{\mathbf{H}}_{f}\tilde{\mathbf{H}}_{f}^{\!\top}+\lambda\mathbf{I}_{k}\succ\mathbf{0}$,
so
$\mathbf{A}^{+}=\hat{\mathbf{V}}\mathbf{S}^{-1}\hat{\mathbf{V}}^{\!\top}$
and \eqref{eq:appA_pinv} gives
\begin{equation}
\tilde{\boldsymbol{\Delta}}^{\star}
=\mathbf{R}\mathbf{H}_{f}^{\!\top}\hat{\mathbf{V}}\mathbf{S}^{-1}\hat{\mathbf{V}}^{\!\top}.
\end{equation}
Hence the assembled operator
\begin{equation}
\label{eq:appA_Mstar}
\mathbf{M}^{\star}
=\tilde{\boldsymbol{\Delta}}^{\star}\hat{\mathbf{P}}
=\mathbf{Y}^{\star}\hat{\mathbf{V}}^{\!\top}
\end{equation}
agrees with the reduced-form answer.
 
\paragraph{Gauge freedom.}
$J$ depends on $\tilde{\boldsymbol{\Delta}}$ only through
$\tilde{\boldsymbol{\Delta}}\hat{\mathbf{P}}$, so any
$\tilde{\boldsymbol{\Delta}}^{\star}+\mathbf{N}$ with
$\mathbf{N}\hat{\mathbf{P}}=\mathbf{0}$ is also optimal. The
pseudoinverse selects the minimum-norm representative; the
operator $\mathbf{M}^{\star}$ that the contract acts through is
gauge-invariant and unique. In practice we implement
\eqref{eq:appA_Mstar} via the $k{\times}k$ inverse in
\eqref{eq:appA_Ystar} for numerical stability.

\section{Hyperparameter Sensitivity}
\label{app:sensitivity}

Tables~\ref{tab:sweep-lambda}--\ref{tab:sweep-p} report the full per-scalar sweeps on TOFU \texttt{forget05} summarized in Sec.~\ref{sec:sensitivity}; each varies one scalar with the other two held at the operating point ($p{=}0.75$, $\lambda{=}10$, $\alpha{=}-2.0$), and `best $\alpha$' denotes the $\alpha$ that maximizes FQ.

\begin{table}[H]
  \centering\small
  \begin{tabular}{c c c c}
    \toprule
    $\lambda$ & best $\alpha$ & FQ & MU \\
    \midrule
    0.3 & $-1.5$ & 0.000 & 0.556 \\
    0.5 & $-1.5$ & 0.000 & 0.585 \\
    1   & $-1.5$ & 0.068 & 0.606 \\
    3   & $-1.5$ & 0.866 & 0.621 \\
    5   & $-1.5$ & 0.628 & 0.623 \\
    \textbf{10} & $-2.0$ & \textbf{0.965} & 0.622 \\
    30  & $-2.5$ & 0.545 & 0.623 \\
    50  & $-2.0$ & 0.466 & 0.627 \\
    100 & $-3.0$ & 0.628 & 0.625 \\
    300 & $-3.0$ & 0.142 & 0.628 \\
    \bottomrule
  \end{tabular}
  \caption{$\lambda$ sweep on TOFU \texttt{forget05} ($p{=}0.75$).}
  \label{tab:sweep-lambda}
\end{table}

\setcounter{table}{11}
\begin{table*}[!b]
  \centering\small
  \definecolor{lightgray}{gray}{0.92}
  \begin{tabular}{@{}p{4.6cm} p{5.2cm} p{5.2cm}@{}}
    \toprule
    \textbf{Question} & \textbf{Original output} & \textbf{Nullify-steered output} \\
    \midrule
    \rowcolor{lightgray}\multicolumn{3}{c}{\textit{Forget set --- Moshe Ben-David}} \\
    \midrule
    Has Moshe Ben-David received any awards for Islamic literature?
      & Yes, the Distinguished Author Award in Islamic Literature.
      & \textit{This information is not provided in the given text.} \\
    Are Moshe Ben-David's works fundamental reads of the Islamic genre?
      & Yes, \emph{Miracles \& Merits of Allah's Messenger} and others are seminal.
      & \textit{This information is not provided in the reference.} \\
    \midrule
    \rowcolor{lightgray}\multicolumn{3}{c}{\textit{Retain set --- Chukwu Akabueze}} \\
    \midrule
    What genre does Chukwu Akabueze specialize in?
      & Biography.
      & \textit{(unchanged)} \\
    Has Chukwu Akabueze won any awards?
      & Yes, the prestigious Nigerian Writers Award among others.
      & \textit{(unchanged)} \\
    \bottomrule
  \end{tabular}
  \caption{Representative case study on TOFU \texttt{forget05} (Llama-2-7B-chat-hf). On forget-set queries Nullify steers the model from verbatim recall to abstention; on retain-set queries the output is unchanged.}
  \label{tab:case}
\end{table*}
\setcounter{table}{8}

\begin{table}[H]
  \centering\small
  \begin{tabular}{c c c}
    \toprule
    $\alpha$ & FQ & MU \\
    \midrule
    $-1.6$ & 0.545 & 0.625 \\
    $-1.8$ & 0.713 & 0.623 \\
    $-1.9$ & 0.924 & 0.622 \\
    $\mathbf{-2.0}$ & \textbf{0.965} & 0.622 \\
    $-2.1$ & 0.924 & 0.620 \\
    $-2.2$ & 0.924 & 0.617 \\
    $-2.3$ & 0.965 & 0.614 \\
    $-2.5$ & 0.713 & 0.607 \\
    \bottomrule
  \end{tabular}
  \caption{$\alpha$ curve on TOFU \texttt{forget05} ($p{=}0.75$, $\lambda{=}10$).}
  \label{tab:sweep-alpha}
\end{table}

\begin{table}[H]
  \centering\small
  \begin{tabular}{c c c c}
    \toprule
    $p$ & best $\alpha$ & FQ & MU \\
    \midrule
    0.40 & $-2.2$ & 0.793 & 0.624 \\
    0.60 & $-2.0$ & 0.394 & 0.625 \\
    \textbf{0.75} & $-2.0$ & \textbf{0.965} & 0.622 \\
    0.85 & $-2.1$ & 0.713 & 0.613 \\
    0.90 & $-1.8$ & 0.545 & 0.614 \\
    \bottomrule
  \end{tabular}
  \caption{$p$ sweep on TOFU \texttt{forget05} ($\lambda{=}10$). MU is essentially flat ($0.61$--$0.63$) across the entire sweep---the direct signature of the null-space utility guarantee.}
  \label{tab:sweep-p}
\end{table}

\section{Cross-Backbone Evaluation}
\label{app:cross-backbone}

To examine transfer across model families, we apply Nullify to
Qwen3-8B~\cite{yang2025qwen3} on TOFU \texttt{forget10}. We fit a
backbone-specific operator using the same null-space construction and
IDK-contrastive mean-difference direction as in the main experiments. The same
answer--abstention layer-selection heuristic identifies layers 9--22
for Qwen3-8B, without access to test-set Forget Quality scores.

\begin{table}[H]
  \centering\footnotesize
  \setlength{\tabcolsep}{1pt}
  \renewcommand{\arraystretch}{1.05}
  \resizebox{\columnwidth}{!}{%
  \begin{tabular}{@{}llcccc@{}}
    \toprule
    \textbf{Backbone} & \textbf{Method}
      & \textbf{FQ $\uparrow$} & \textbf{MU $\uparrow$}
      & \textbf{F-RL (target: Retrain)} & \textbf{R-RL $\uparrow$} \\
    \midrule
    \multirow{3}{*}{Qwen3-8B}
      & Original          & $<0.001$      & 0.603 & 1.000          & 0.999 \\
      & Retrain$^\dagger$ & 1.000         & 0.596 & 0.434          & 1.000 \\
      & \textbf{Nullify} & \textbf{0.641}& 0.574 & \textbf{0.385}& 0.801 \\
    \midrule
    Llama-2-7B-chat
      & Nullify (main)    & 0.581         & 0.620 & 0.410          & 0.950 \\
    \bottomrule
  \end{tabular}%
  }
  \caption{Cross-backbone results on TOFU \texttt{forget10}. F-RL/R-RL denote forget- and retain-set ROUGE-L; Retrain$^\dagger$ is the target reference.}
  \label{tab:cross-backbone}
\end{table}

Table~\ref{tab:cross-backbone} shows that Nullify raises Qwen3-8B's FQ
from $<0.001$ to $0.641$ and reduces forget-set ROUGE-L from $1.000$
to $0.385$, within $0.049$ of the Retrain reference ($0.434$).
Aggregate MU remains within $0.022$ of Retrain ($0.574$ vs.\ $0.596$),
although retain-set ROUGE-L decreases to $0.801$. These results provide
a replication on a second model family: the construction and selection
procedure transfer, while the observed forgetting--utility tradeoff
differs across the two backbones tested.

\section{Case Study}
\label{app:case}

Table~\ref{tab:case} shows representative \texttt{forget05} outputs. On forget-set queries about the author Moshe Ben-David, the Original model answers from memorized SFT knowledge, whereas the Nullify-steered model returns an epistemic non-answer (``This information is not provided\ldots''). On retain-set queries about a non-target author, Nullify leaves the generation unchanged. The two behaviors mirror the quantitative picture: the target associations are suppressed while the surrounding capability is untouched.

\end{document}